\documentclass{article}

\usepackage[round]{natbib}
\usepackage[letterpaper,textheight=9in,textwidth=5.5in,top=1in]{geometry}
\usepackage{authblk}

\usepackage[utf8]{inputenc} 
\usepackage[T1]{fontenc}    
\usepackage{hyperref}       
\usepackage{url}            
\usepackage{booktabs}       
\usepackage{amsfonts}       
\usepackage{nicefrac}       
\usepackage{microtype}      

\usepackage{crmath}
\usepackage{amsmath}
\usepackage{mathtools}
\usepackage{bbm}

\usepackage{color}

\usepackage{amsthm}

\definecolor{orange}{RGB}{255,153,0}
\definecolor{green}{RGB}{9,112,84}
\definecolor{yellow}{RGB}{52,100,100}
\definecolor{blue}{RGB}{101,153,255}

\newtheorem{thm}{Theorem}

\newtheorem{lemma}{Lemma}
\newtheorem{corollary}{Corollary}

\usepackage{times}
\usepackage{graphicx} 
\usepackage{subfigure}

\usepackage[ruled,vlined,linesnumbered]{algorithm2e}

\title{Causal Compression}

\author{Aleksander Wieczorek}
\author{Volker Roth}
\affil{Department of Mathematics and Computer Science, University of Basel, Switzerland\\
	\texttt{\{aleksander.wieczorek, volker.roth\}@unibas.ch}}

\begin{document}

\maketitle

\begin{abstract}
	We propose a new method of discovering causal relationships in temporal data based on the notion of causal compression. To this end, we adopt the Pearlian graph setting and the directed information as an information theoretic tool for quantifying causality. We introduce chain rule for directed information and use it to motivate causal sparsity. We show two applications of the proposed method: causal time series segmentation which selects time points capturing the incoming and outgoing causal flow between time points belonging to different signals, and causal bipartite graph recovery. We prove that modelling of causality in the adopted set-up only requires estimating the copula density of the data distribution and thus does not depend on its marginals. We evaluate the method on time resolved gene expression data.

\end{abstract}

\section{Introduction}
\label{sec:introduction}

Causality modelling has recently received much attention in the machine learning community. Various approaches to discovering causal relationships in the data have been proposed. The idea of an intervention, i.e.\ forcing a node in a graphical model to a particular value and analysing the resulting distribution has been introduced in~\citep{book:pearl:causality} and further developed in~\citep{art:eberhardt:interventions_causality, art:hagmayer:interventions_causality} as well as adjusted to account for observational data~\citep{art:buhlnamm:interventional_observational}. The use of structural equation models has been advocated in~\citep{report:pearl:sem_causality, thesis:peters:sem_causality}. Expressing Pearl's intervention calculus in terms of information theoretic concepts capturing the difference between interventional and observational distributions resulted in rich literature, e.g.\ \citep{art:raginsky:directed_inf_causality, art:amblard:granger, art:massey:directedInformation}. As a result, asymmetrical information theoretic measures \citep{art:massey:directedInformation} are used for modelling causal relationships in graphical models.

We build on these ideas by employing directed information between time series to quantify the amount of directed information flow. In this setting, we introduce the notion of \textbf{causal compression}, i.e.\ compression which, by maximising directed information, selects time points carrying the causal flow between the time series. We show that sparsity of compression ensures causal compression, i.e.\ only the nodes or edges which reflect directed causal connections are selected. We then construct a constrained optimisation problem for finding a causal sparse representation of a time series. We also show that the modelled directed relationships only depend on the copula.

\paragraph*{Motivation}

We motivate our compression-based approach by a general principle of solving problems formulated by Vapnik in the context of learning theory in~\citep{book:vapnik}: ,,
When solving a problem, try to avoid solving more general problem as an intermediate step.`` We interpret it in the following manner: it is not necessary to infer a general structure $G$ from data $D$, if one is only interested in a function $f(D)$ preserving only certain semantics of $G$ (see Figure~\ref{pic:problem_setting}).
In our setting, the general structure $G$, the estimation of which we try to avoid, is the full causal network. We show that the partial semantics defined by $f(G)$ can be obtained by employing causal compression without inferring $G$. We give two examples of $f(D)$, presented schematically in Figure~\ref{pic:motivation}. The first $f(D)$ is the \textbf{causal segmentation} of time points of one time series into those that exhibit outgoing or incoming causal flow (\textcolor{orange}{orange} and \textcolor{green}{green} nodes in Figure~\ref{pic:motivation}, respectively) to the other time series and those involved in instantaneous information exchange (\textcolor{blue}{blue} nodes in Figure~\ref{pic:motivation}). Another example of $f(D)$ that we present is \textbf{causal bipartite graph estimation}, e.g.\ computing a mixed bipartite graph between the two time series, where the arrows mean causal dependence and edges mean instantaneous information exchange.






\begin{figure*}
\centering     
\subfigure[Problem setting]{ \label{pic:problem_setting}\includegraphics[width=0.4\textwidth]{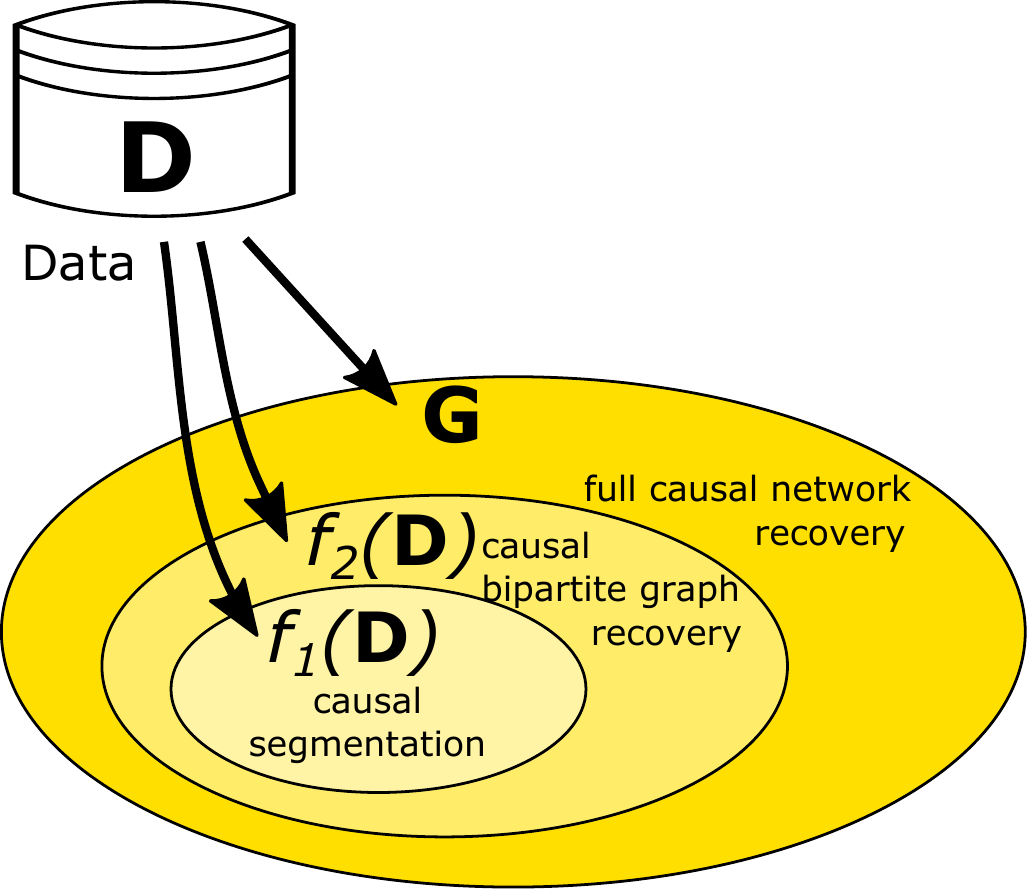}}
\quad \quad \quad \quad
\subfigure[Direct computation of $f_1(D), f_2(D)$]{\label{pic:motivation}\includegraphics[width=0.4\textwidth]{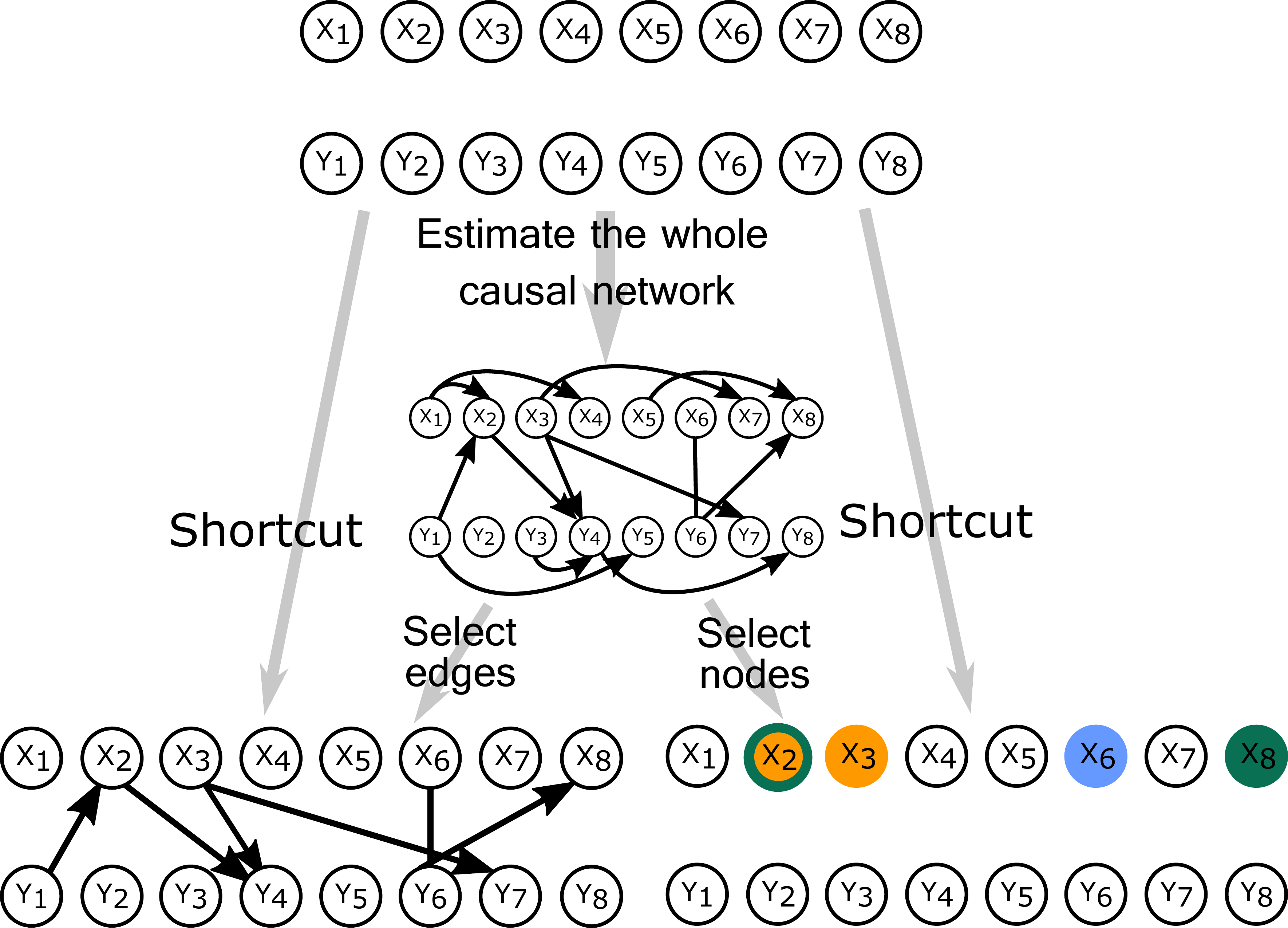}}
\caption{Causal compression --- motivation}
\end{figure*}

We show how to compute the two $f(D)$ for two parallel time series, which might describe two evolving systems, but all the concepts can be extended to more general structures according to~\citep{art:raginsky:directed_inf_causality} as long as a global ordering of the random variables is known.




This paper is structured as follows: Section~\ref{sec:relatedwork} formalises the setting of Pearlian graphs as well as introduces the formalism of directed information theory, both for directed acyclic graphs and for time series. In Section~\ref{sec:model}, the proposed causal compression principle is described, its applications to causal segmentation of time series and to directed bipartite graph recovery are characterised and the copula formulation of causal discovery is presented. Section~\ref{sec:experiments} includes experiments on synthetic as well as gene expression data. Concluding remarks are presented in Section~\ref{sec:conclusion}.


\section{Related work}
\label{sec:relatedwork}

\paragraph*{Causal graphs.}

Causal relationships in graphical models are frequently represented with directed acyclic graphs (DAGs). The arrows in the graphs can be imbued with causal interpretation in different ways, e.g.\ through considering differences in factorisations of the joint probability~\citep{art:causality:factorisations} or by introducing the \textit{causal Markov condition}~\citep{art:causality:causalMarkovCondition}. We follow the approach to representing causality with DAGs proposed in~\citep{book:pearl:causality}. It requires that one be able to perform, or think of performing, an \textit{intervention} on any node or collection of nodes in the graph. An intervention means that the variable intervened upon has its value set externally, while the influence of any other variables in the DAG (most importantly its parents) upon it is suppressed. This process corresponds to measuring the influence of a chosen set of variables on the rest of the system. A \textit{Pearlian DAG} satisfies two more conditions:
\begin{enumerate}
	\item It represents the conditional independence relations of the underlying probability distribution via \textit{d}-separation, i.e.\ any pair of sets of variables \textit{d}-separated by a third set $U$ are conditionally independent given $U$.
	\item For any node $V$, its conditional distribution given its parents does not depend on interventions on any other nodes in the DAG.
\end{enumerate}
The latter condition is called \textit{modularity} and can be shown to imply \textit{locality}, i.e.\ that only interventions performed on the parents of $V$ affect $V$. This can be thought of as an intuitive extension of conditional independence representation in graphical models to causality: the \textit{absence of an arrow} between two nodes implies the \textit{absence of a direct causal relationship} between them. Thus, Pearlian DAGs, while representing conditional independence relationships in the data, also provide a setting for analysing direct causal relationships.

\paragraph*{Causal graphs and directed information.}




Let $(V,E)$ be a Pearlian graph where the elements of $V$ are random variables taking values in $\mathbb{R}^{n}$ and assume $X,Y \subseteq V$.

Define the \textit{Kullback-Leibler divergence} between two (discrete or continuous) probability distributions $P$ and $Q$ as $D_{KL} (P || Q) = \mathbb{E}_{P} \log \frac{P(X)}{Q(X)}$
and the \textit{conditional Kullback-Leibler divergence} as $D_{KL} (P_{Y|X} || Q_{Y|X} |P_{X}) = \mathbb{E}_{P_{X,Y}} \log \frac{P(Y|X)}{Q(Y|X)}$
The \textit{mutual information} between $X$ and $Y$ is then defined as $I(X;Y) = D_{KL}(P(X,Y)||P(X)P(Y))$.

For any disjoint $X$ and $Y$ denote with $P_{X|\mbox{do}(Y=y)}$ the \textit{interventional distribution} of $X$, i.e.\ the distribution of $X$ which results from intervening on $Y$ by setting its value to $y$ as described above. This distribution is contrasted with the \textit{observational distribution} of $P_{X|Y}$ which is obtained by passively observing the values of $V$.

The idea of defining causality as the difference between the two distributions has recently gained popularity \citep{art:massey:directedInformation,art:raginsky:directed_inf_causality,art:amblard:granger}. It was originally introduced in~\citep{art:massey:directedInformation} as \textit{causal conditioning} of time series and extended to the notion of \textit{directed stochastic kernels} in~\citep{art:tatikonda:dst}. We follow this approach and define the \textit{directed information} as:
\begin{equation}
I(X \rightarrow Y) = D_{KL} (P_{X|Y} || P_{X|\mbox{do}(Y)}|P_{Y}) = \mathbb{E}_{P_{X,Y}} \log \frac{P(X|Y)}{P(X|\mbox{do}(Y))}.
\label{eq:dmi:general}
\end{equation}

This measure has an intuitive interpretation in the setting of interventions in Pearlian graphs: if its value is small, then the two distributions are similar, thus any common changes of $X$ and $Y$ can be identified without intervening on $Y$. Otherwise, performing an intervention on $Y$ has influenced the distribution of $X$, hence the difference must stem from the connections between $X$ and $Y$ in $V$, which were destroyed while intervening on $Y$.

The directed information as defined in Eq.~(\ref{eq:dmi:general}) quantifies the difference between the interventional distribution of $X|\mbox{do}(Y)$ and the observational conditional distribution of $X|Y$. One might also consider a mixed quantity, i.e.\ an interventional distribution with conditioning on a set of passive observations. This leads to the definition of conditional directed information for three disjoint sets $X,Y,Z \subseteq V$~\citep{art:raginsky:directed_inf_causality}:
\begin{equation}
I(X \rightarrow Y|Z) = D_{KL} (P_{X|Y,Z} || P_{X|\mbox{do}(Y),Z}|P_{Y,Z}) = \mathbb{E}_{P_{X,Y,Z}} \log \frac{P(X|Y,Z)}{P(X|\mbox{do}(Y),Z)}.
\label{eq:cdmi:general}
\end{equation}
Analogously to directed information, conditional directed information as defined in~Eq.~(\ref{eq:cdmi:general}) can be interpreted as a measure of the causal relationship between $X$ and $Y$, when paths traversing $Z$ in the underlying DAG are excluded.

\paragraph*{Directed Information Theory for Time Series.}

For the sequel, assume $X^{n} = (X_1, X_2, \dots, X_n)$ and $Y^{n} = (Y_1, Y_2, \dots, Y_n)$ to be random vectors representing time series indexed at the same time points.

For a Pearlian DAG $(V,E)$ where $V$ consists of two time series $X^{n}, Y^{n}$, and $E$ of all possible arrows pointing to the future (i.e.\ all arrows $(X_i,X_j)$, $(Y_i,Y_j)$, $(X_i,Y_j)$, $(Y_i,X_j)$ with $i<j$), the directed information defined in Eq.~(\ref{eq:dmi:general}) takes the following form (see~\citep{art:massey:directedInformation,art:amblard:granger}):
\begin{equation}
I(X^{n} \rightarrow Y^{n}) = \sum_{i=1}^{n} I(X^{i}; Y_{i}|Y^{i-1}),
\label{eq:dmi}
\end{equation}
where $I(X;Y|Z) = D_{KL}(P(X,Y,Z)||P(X|Z)P(Y|Z)P(Z))$.
As in~\citep{art:amblard:granger}, $X^{n-1}$ or $Y^{n-1}$ stand for delayed collections of samples of $X^{n}$ or $Y^{n}$ and their first elements should be understood as wild cards not influencing the conditioning. This means that the symbol $X^{n-1}$ denotes an $n$-dimensional time series $(*, X_1, X_2, \dots , X_{n-1})$, where the first element $*$ does not affect conditioning and where $I(*; Z) = 0$ for any $Z$.

In~\citep{art:amblard:granger} it is also shown that the decomposition of mutual information into directed informations and the \textit{instantaneous coupling} term $I(X^{n} \leftrightarrow Y^{n})$ exist:
\begin{equation}
I(X^{n}; Y^{n}) = I(X^{n-1} \rightarrow Y^{n})+ I(Y^{n-1} \rightarrow X^{n}) + I(X^{n} \leftrightarrow Y^{n}),
\label{eq:dmi:decomposition}
\end{equation}
where $I(X^{n} \leftrightarrow Y^{n}) = \sum_{i=1}^{n} I(X_{i}; Y_{i}|Y^{i-1}, X^{i-1})$.



For jointly Gaussian distributed $(X^{n},Y^{n})$, let the partitioning of the joint covariance matrix $\Sigma_{X^{n},Y^{n}}$ be denoted as follows:
$(X^{n}, Y^{n}) \sim \mathcal{N}  \left(\begin{smallmatrix}
\Sigma_{X^{n}} & \Sigma_{X^{n}Y^{n}}\\
{\Sigma_{Y^{n}X^{n}}} & \Sigma_{Y^{n}}\\
\end{smallmatrix}\right)$.
Recall that for $n$-dimensional Gaussian distributed random variables, entropy, and hence mutual information, have the following form (for the sake of clarity we will neglect the constant term $(2\pi e)^n$ in the sequel):
\begin{equation}
I(X^{n}; Y^{n}) = h(X^{n}) - h(X^{n}|Y^{n}) = \frac{1}{2}\log \left( (2\pi e)^n |\Sigma_{X^{n}}| \right) - \frac{1}{2}\log \left( (2\pi e)^n |\Sigma_{X^{n}|Y^{n}}| \right),
\label{eq:gmi}
\end{equation}
where $\Sigma_{X^{n}|Y^{n}}$ denotes the covariance matrix of $X^{n}|Y^{n}$. Hence, for jointly Gaussian distributed $X^{n}$ and $Y^{n}$ we have:
\begin{equation}
I(X^{n-1} \rightarrow Y^{n}) = \sum_{i=1}^{n-1} ( \log |\Sigma_{X^{i}|Y^{i}}| - \log|\Sigma_{X^{i}|Y^{i+1}}| ).
\label{eq:gdmi}
\end{equation}


There have been attempts at discovering causal relationships in time series. \textit{Transfer entropy}, which is the asymptotic value of the directed information when stationarity of the time series is assumed, is used to model neurobiological data~\citep{art:biol:neuro1,art:biol:neuro2}. 
In~\citep{art:quinn:graphs,art:quinn:trees}, directed information is employed to measure causal relationships between nodes in networks of stochastic processes by approximating the probability distribution of the graphical model with causal trees or causal graphs. Unlike those approaches, we do not treat time series (or random processes) as nodes and do not model relationships between such nodes. Rather than focusing on the interdependencies among multiple time series, we propose to model causal relationships between specific time points of the time series. We also do not have to make any assumptions concerning stationarity.

Another approach to causal compression could be taken where a series of statistical tests of conditional mutual informations is devised in order to compute the directed information (according to Eq.~(\ref{eq:dmi})). In this set-up, however, all subsets of the time series would have to be tested in order to establish the optimal representation. In contrast to this, the proposed method produces a solution path where nodes comprising the compressed representation are added as the sparsity criterion is relaxed.

\section{Causal Compression}
\label{sec:model}

We propose to combine directed information theory as introduced in Section~\ref{sec:relatedwork} with sparse optimisation in order to compute causal compression of time series as motivated in Section~\ref{sec:introduction}. In this way, directed relationships in the data can be modelled. To this end, we note that sparse representation ensures causal compression, i.e.\ that for a given value of directed information, choosing the most sparse time series representation is equivalent to excluding the nodes that do not contribute to the direct causal relationships in the Pearlian graph (see Corollary~\ref{thm:cc}). Let $X^{n}$ and $Y^{n}$ represent the two time series. The time series $T^{n}$, which is the result of the compression, is a sparse representation of $X^{n}$ that preserves a given amount of the directed information between $X^{n}$ and $Y^{n}$. This representation is obtained by introducing the \textbf{causal compression} principle, where the sparse $T^{n}$ maximising the \textbf{directed} information is found. We subsequently (Sections~\ref{sec:model:sub:cs}~and~\ref{sec:model:sub:bipartite}) present two applications of causal compression. Both applications are ways of circumventing the necessity of estimating the whole causal network, as presented in Figure~\ref{pic:motivation}. Finally, we show that solutions to both problems only depend on the copula density of $(X^n,Y^n)$.



The property of sparsity as a building block of causal compression is formalised as Corollary~\ref{thm:cc}. We first show that the equivalent of the chain rule for mutual information also holds for directed information (Lemma~\ref{thm:chain}), and subsequently apply it to time series $X^{n}$, $Y^{n}$.

\begin{lemma}[Chain rule for directed information]
	\label{thm:chain}
	For any disjoint sets $A,B,C$
	\begin{equation}
	I(A,B \rightarrow C) = I(A \rightarrow C) + I(B \rightarrow C|A)
	\end{equation}
\end{lemma}
\begin{proof}
The proof is similar to the one of chain rule for conditional Kullback-Leibler divergence and follows from the factorisation of the underlying probability distribution (see Eq.~(\ref{eq:dmi:general})):
\begin{equation*}
\begin{aligned}
&I(A,B \rightarrow C) = \mathbb{E} \log \frac{P_{A,B|C}(A,B|C)}{P_{A,B|\mbox{do}(C)}(A,B|\mbox{do}(C))} \\
&= \mathbb{E} \log \frac{P_{A|C}(A|C)P_{B|C,A}(B|C,A)}
{P_{A|\mbox{do}(C)}(A|\mbox{do}(C))P_{B|\mbox{do}(C),A}(B|\mbox{do}(C),A)}\\
&= \mathbb{E} \log \frac{P_{A|C}(A|C)}
{P_{A|\mbox{do}(C)}(A|\mbox{do}(C))} +\mathbb{E} \log \frac{P_{B|C,A}(B|C,A)}
{P_{B|\mbox{do}(C),A}(B|\mbox{do}(C),A)} = I(A \rightarrow C) + I(B \rightarrow C|A),
\end{aligned}
\end{equation*}
where all expectations are taken with respect to $P_{A,B,C}$.
\end{proof}

\vspace{1mm}
\begin{corollary}[Causal compression is equivalent to sparsity]
\label{thm:cc}
For $A,B \subset X^n$, $A\cap B = \emptyset$
\begin{equation}
I(A,B \rightarrow Y^{n}) = I(A \rightarrow Y^{n}) \quad \Leftrightarrow \quad I(B \rightarrow Y^{n}|A) = 0
\end{equation}
\end{corollary}
Corollary~\ref{thm:cc} states that for the same value of directed information between a subset $S$ of $X^n$, and $Y^n$, adding more variables to the subset $S$ means adding variables which do not exhibit causal (in the sense of Pearlian graphs) relations with $Y^{n}$ other than via the original $S$. Therefore, the optimal causal compression at a given level of directed information is ensured by the sparsity of the compressed representation of $X^n$, i.e.\ by selecting as few time points as possible. Note that Corollary~\ref{thm:cc} can be interpreted in the spirit of Granger causality: the variables in $X^n$ that are not selected by the sparsity requirement do not Granger-cause the effect $Y^n$.

Based on Corollary~\ref{thm:cc}, the idea behind the causal compression principle is to find the compressed representation (here, $T^{n}$) of a set of random variables by maximising directed information involving $T^{n}$ and enforcing its sparsity. The causal compression principle can therefore be implemented in any method that:
\begin{enumerate}
	\item admits the use of directed information or cognate information theoretic tools,
	\item allows for incorporation of sparsity.
\end{enumerate}
We now proceed to describe two ways of applying the causal compression as depicted in Figure~\ref{pic:motivation} by assuming the time series $X^{n}$, $Y^{n}$ to be jointly Gaussian distributed and devising an optimisation problem\footnote{Note that other approaches can be proposed for implementing the causal compression principle, such as adjusting the sparse Gaussian information bottleneck~\citep{art:rey:sgib} for preserving $I(T^{n-1} \rightarrow Y^{n})$.} for finding the optimal sparse representation of $X^{n}$. We then relax the Gaussian assumption in Section~\ref{sec:model:sub:copula}. We begin with specifying and solving the optimisation problem which implements the causal compression principle in Section~\ref{sec:model:sub:algo}.

\subsection{Defining and solving the optimisation problem}
\label{sec:model:sub:algo}

According to the conditions specified in Section~\ref{sec:model}, we define $T^{n}$, the compressed version of $X^{n}$, to be a linear noisy projection of $X^{n}$, i.e.\ $T^{n} = A X^{n}+\xi$ with $\xi \sim \mathcal{N}(0,I)$. In order to impose sparsity of the projection, we assume $A$ to be diagonal. Thus, non-zero entries of $A$ define which elements of $X^{n}$ are chosen to the sparse representation. Note that if the projection were not noisy, the optimisation problem would reduce to binary feature selection akin to the statistical tests approach considered in Section~\ref{sec:relatedwork} (the directed information between $T^{n}$ and $Y^{n}$ would then only depend on the rank of $A$). We incorporate directed information by maximising its value according to the decomposition given by Eq.~(\ref{eq:dmi:decomposition}). The assumption that $X^{n}$ is Gaussian distributed means that $T^{n}\sim \mathcal{N}(0,A \Sigma_{X^{n}} A^\top + I)$. Thus, the optimisation problem for finding the causal compression of $X^{n}$ described in Section~\ref{sec:model} can be stated in a LASSO fashion as follows:
\begin{equation}
\underset{D=A^\top A}{\mbox{min}}  -I(AX^{n-1} \rightarrow Y^{n}) \quad \mbox{s.\ t.}\quad ||D||_1 < \kappa
\label{eq:opt:general}
\end{equation}
where $D=A^\top A$ is a diagonal matrix and $||D||_1 = \sum_{i=1}^{n}|d_i|$ its $L_1$ norm. Plugging Eq.~(\ref{eq:gdmi}) in to Eq.~(\ref{eq:opt:general}) and noting that $|AMA^\top+I|=|MA^\top A+I|$, yields:
\begin{equation}
\begin{aligned}
\underset{D}{\mbox{min}}  - \sum_{i=1}^{n-1}  &( \log |\Sigma_{X^{i}|Y^{i}} D_{i} + I| -\log|\Sigma_{X^{i}|Y^{i+1}} D_{i} + I| ) \\
&\mbox{s.\ t.}\quad  \sum_{i=1}^{n}|d_i| < \kappa \mbox{ \quad and \quad} \forall_{i}\  d_i \geq 0
\label{eq:opt:general:gauss}
\end{aligned}
\end{equation}
where $D_i = \mbox{diag}({d_1}, \dots , {d_i})$.

Greedy optimization methods such as \textit{stagewise forward} \citep{tibshirani2014general} can now be applied to approximate the optimal solution to~(\ref{eq:opt:general}). The stagewise forward procedure will recover the whole solution path. For handling the non-negativity constraints on the elements of $D$, we use gradient projection in the spirit of the \textit{monotone stagewise forward} method~\citep{hastie2007forward}. This procedure is formalised as Algorithm~\ref{alg:1}.

The gradient $g$ computed in line~\ref{alg:1:gradient} of Algorithm~\ref{alg:1} is a sum of terms of the following form: $[(\Sigma_{X^{i}|Y^{i}}D_i+I)^{-1}\cdot D_i]\mathbbm{1}$ for every $\Sigma_{X^{i}|Y^{i}}$ and $\Sigma_{X^{i}|Y^{i+1}}$ between $i=1$ and $n-1$. By applying the Sherman–Morrison formula for rank-$1$ update $2(n-1)$ times, the gradient can be computed in $O(n^3)$ time in every iteration, assuming that the covariance matrix $\Sigma_{X^{n},Y^{n}}$ has been precomputed. The while loop is executed at most $\frac{\kappa}{\epsilon}$ times, where $\kappa$ is the sparsity parameter and $\epsilon$ --- the learning rate. Thus, the running time of Algorithm~\ref{alg:1} is $O(\frac{\kappa}{\epsilon} n^3)$



\begin{algorithm}[t]
	\KwIn{Sample covariance matrix $\Sigma_{X^{n},Y^{n}}$, learning rate $\epsilon$}
	\KwOut{$D = \mbox{diag}(d_1, \dots ,d_n)$ }
	Initialise $D = \mbox{diag}(0, \dots , 0)$\\
	\While{$\sum_{i=1}^{n} |d_i| < \kappa$} {
		$ g = \nabla \sum_{i=1}^{n-1} ( \log |\Sigma_{X^{i}|Y^{i}} {D_i} + I| -\log|\Sigma_{X^{i}|Y^{i+1}} D_i + I| )$\label{alg:1:gradient}\\
		\If{$\max_{i} g_i \leq 0$}{break}
		$j=\arg\max_{k} g_k$\\
		$d_j = d_j + \epsilon$\\
	}
	\caption{Optimisation algorithm for Eq.~(\ref{eq:opt:general:gauss})}
	\label{alg:1}
\end{algorithm}

\subsection{Causal Segmentation}
\label{sec:model:sub:cs}

In this section we show how the causal compression principle can be used to classify points in a time series into three classes with respect to another time series: points carrying incoming directed information, outgoing directed  information and points instantaneously coupled with corresponding points from the other time series.

The above optimisation problem finds a set $X^{out}$, which is a compressed representation of $X^{n}$. This compressed representation (i.e.\ non-zero values in the vector $(d_1, \dots , d_n)$ in Eq.~(\ref{eq:opt:general})) defines the segment of $X^{n}$ containing all the nodes in $X^{n}$ that carry directed information from $X_j$ to $Y_k$ with $j<k$ (\textcolor{orange}{orange} nodes in Figure~\ref{pic:motivation}). Thus all nodes in $X^{n}$ with possible outgoing causal flow to future nodes in $Y^{n}$ are selected.

As defined in Eq.~(\ref{eq:dmi:decomposition}), mutual information between the compressed representation of $X^{n}$ (i.e.\ $T^{n}$) and $Y^{n}$ decomposes into three elements: $I(Y^{n};T^{n}) = \textcolor{green}{I(Y^{n-1} \rightarrow T^{n})} + \textcolor{orange}{I(T^{n-1} \rightarrow Y^{n})} + \textcolor{blue}{I(T^{n} \leftrightarrow Y^{n})}$. This means that by substituting the directed information $I(T^{n-1} \rightarrow Y^{n})$ with $I(Y^{n-1} \rightarrow T^{n})$  in Eq.~(\ref{eq:opt:general}) and solving the resulting optimisation problem, the compressed representation $T^{n}$ of $X^{n}$ is forced to contain all nodes from $X^{n}$ carrying the information flow in the other direction, i.e.\ from $Y_j$ to $X_k$ with $j<k$. In this way the subset $X^{in}$ of $X^{n}$ with all the nodes in $X^{n}$ with possible incoming causal flow from past nodes in $Y^{n}$ is selected (see \textcolor{green}{green} nodes in Figure~\ref{pic:motivation}). Analogously, if $I(T^{n-1} \rightarrow Y^{n})$ is replaced with $I(T^{n} \leftrightarrow Y^{n})$ in Eq.~(\ref{eq:opt:general}), one obtains the set $X^{eq}$, i.e.\ nodes in $X^{n}$ which are instantaneously coupled with their counterparts in $Y^{n}$ (\textcolor{blue}{blue} nodes in Figure~\ref{pic:motivation}).


The above procedure can be summarised as follows: in order to fully describe the causal relationships involving $X^{n}$, find the segments in $X^{n}$ containing nodes with outgoing, incoming or instantaneous causal flow. To this end, compress $X^{n}$ to $T^{n}$ three times, each time modifying the condition in~(\ref{eq:opt:general}) accordingly:
\begin{itemize}
\item optimise $\textcolor{orange}{I(T^{n-1} \rightarrow Y^{n})}$ to select $X^{out}$: the segment of $X^{n}$ with outgoing causal flow to the future of $Y^{n}$,
\item optimise $\textcolor{green}{I(Y^{n-1} \rightarrow T^{n})}$ to select $X^{in}$: the segment of $X^{n}$ with incoming causal flow from the past $Y^{n}$,
\item optimise $\textcolor{blue}{I(T^{n} \leftrightarrow Y^{n})}$ to select $X^{eq}$: the segment of $X^{n}$ with which is instantaneously coupled to $Y^{n}$.
\end{itemize}

\subsection{Causal bipartite graph retrieval}
\label{sec:model:sub:bipartite}

In this section we show how to apply the causal compression principle to estimate the causal bipartite graph between two time series, without estimating the whole directed network. This corresponds to the left hand side short-cut in Figure~\ref{pic:motivation}. Note that it is a different problem than the causal segmentation described in Section~\ref{sec:model:sub:cs}, since it is not sufficient to estimate which points are in the $X^{out}$ and $X^{in}$ sets. It also has to be established, to which points in the other time series the arrows lead. Note that it is straightforward to infer the causal segmentation given the causal bipartite graph, but not the other way around (see Figure~\ref{pic:problem_setting}).

In order to establish the arrows, one can make use of the decomposition of the directed information between $X^n$ and $Y^n$ (Eq.~(\ref{eq:dmi})). It consists of a sum of terms of the form $ I(X^{i}; Y_{i}|Y^{i-1})$ for all $i$, where each such term measures the information flow from the past of $X^i$ to the current $Y_i$. Therefore, by exchanging the expression for $I(T^{n-1} \rightarrow Y^{n})$ with $ I(T^{i}; Y_{i}|Y^{i-1})$ in Eq.~(\ref{eq:opt:general}), one obtains a sparse representation $X^{out}_{i}$ of all time points in $(X_1, \dots, X_{i-1})$ that make up the causal flow to $Y_i$, and thus all arrows that lead to $Y_i$. If this procedure is now repeated for all $i$, all arrows from $X^n$ to $Y^n$ are established. The arrows in the other direction, i.e.\ from $Y^n$ to $X^n$, are established by simply exchanging $X^n$ and $Y^n$ and finding the sparse compression of $Y^n$. The undirected edges representing pairs of instantaneously coupled points can be found as described in Section~\ref{sec:model:sub:cs}, since they always connect time points with the same index $i$.

As in the case of causal segmentation, the above procedure consists of three steps where the causal compression is performed with different optimisation objectives (condition in Eq.~(\ref{eq:dmi})):
\begin{itemize}
\item for each $i \in 2,\dots n$, optimise $I(T^{i}; Y_{i}|Y^{i-1})$ to select $X_{i}^{out}$: the segment of $X^{i-1}$ with outgoing causal flow to $Y_{i}$; add arrows from $X_{i}^{out}$ to $Y_{i}$ to the model,
\item for each $i \in 2,\dots n$, optimise $I(T^{i}; X_{i}|X^{i-1})$ to select $Y_{i}^{out}$: the segment of $Y^{i-1}$ with outgoing causal flow to $X_{i}$; add arrows from $Y_{i}^{out}$ to $X_{i}$ to the model
\item optimise $I(T^{n} \leftrightarrow Y^{n})$ and $I(T^{n} \leftrightarrow X^{n})$ to select $X^{eq}$ and $Y^{eq}$; add edges between all pairs of $X_i$, $Y_i$ for which $X_i \in X^{eq}$ and $Y_i \in Y^{eq}$.
\end{itemize}

The optimisation problem is solved as described in Section~\ref{sec:model:sub:cs}, the only difference being the substitution of the full directed information $I(T^{n-1} \rightarrow Y^{n})$ with its element corresponding to the time point $i$: $ I(T^{i}; Y_{i}|Y^{i-1})$

\subsection{Copula extension}
\label{sec:model:sub:copula}

Directed information, as well as conditional mutual information, can be decomposed into a sum of \textit{multiinformations}~\citep{phd:multiinformation}: $I(X^{n}  \rightarrow Y^{n}) =  \sum_{i=1}^{n} \left[ M(X^{i}, Y^{i})-M(X^{i},Y^{i-1})\right] - M(Y^{n})$, 
where $M(X^{n}) =  D_{KL}(P(X_1, \dots , X_n)||P(X_1)\dots P(X_n))$.
In~\citep{art:mi:copula} it was shown that for a continuous random vector $Z^{m} = (Z_1, \dots , Z_m)$, its multiinformation is equal to the negative entropy of its copula density, i.e.\ $M(Z^{m}) = -H(c_{Z^{m}})$, where $H(c_{Z^{m}}) = \int_{[0,1]^m} \log c_{Z^{m}}(u) c_{Z^{m}}(u) \mbox{d} u$ and $c_{Z^{m}}$ is the copula density of the vector $Z^{m}$.

\begin{thm}[Copula formulation of causal discovery.]
\label{thm:1}
For continuous $(X^{n}, Y^{n})$, any causal relationship described with directed information only depends on the entropy of copula density of $(X^{n}, Y^{n})$.
\footnote{Note that this result reaches beyond the time series setting as long as one expresses directed information as a sum of conditional mutual informations analogously to Eq.~(\ref{eq:dmi}) with conditioning on parent sets.}
\end{thm}

This result can be shown by expressing directed information in terms of multiinformations and using their equivalence to the copula entropy as described above. 

From Theorem~\ref{thm:1} it follows that the causal compression principle, as described in this section, only depends on the copula density of $(X^{n}, Y^{n})$. This means that for inference we only have to estimate the copula part of the distribution. In particular, for Gaussian distributed data only the correlation matrices have to be identified. The Gaussian assumption can therefore be relaxed to the class of distributions with a Gaussian copula, sometimes called meta-Gaussian distributions.

In practice, to fit a semi-parametric copula model (with non-parametric marginals and a parametric Gaussian copula), one has to estimate correlation matrices between the $2n$ dimensions of the model. They depend on the normal scores $\Phi^{-1}(r_{ik})$ where $r_{ik}$ is the rank of the $i$-th observation of dimension $k$. The normal scores rank correlation coefficient between dimensions $k$ and $j$ can then defined as $\frac{\sum_{i=1}^{d}\Phi^{-1}(\frac{r_{ik}}{d})\Phi^{-1}(\frac{r_{ij}}{d})}{\sum_{i=1}^{d}(\Phi^{-1}(\frac{i}{d})) ^2}$, which is an efficient estimator studied in~\citep{art:copula:estimator}. The correlation matrix made up of such coefficients for all dimensions  is positive definite, and is in practice fed in to Algorithm~\ref{alg:1} in place of the covariance matrix $\Sigma_{X^n,Y^n}$ to perform inference on meta-Gaussian distributions.

\section{Experiments}
\label{sec:experiments}

\subsection{Synthetic data}
\label{sec:experiments:sub:synthetic}
We first test our approach on artificial data. To this end we draw $500$ samples from a multivariate Gaussian model for $(X^n, Y^n)$. For $Z := (X_1,\dots,X_n,Y_1,\dots,Y_n)^t$, we assume the model to be  $Z = BZ + \xi$, with $B$ being a lower triangular matrix and $\xi \sim N(0,\sigma^2I)$, i.e.\ $Z = (I-B)^{-1}\xi \sim N(0, \Sigma), \quad \Sigma^{-1} = \sigma^{-2}(I-B)(I-B)^t$. We define $B$ so that this corresponds to a $6$-th order Markov model for $X^n$ and $Y^n$, 3 additional links $X\to Y$  and 4 links $Y\to X$ as well as two instantaneous coupling terms $X\leftrightarrow Y$, as depicted in Figure~\ref{pic:solpath}.

Based on the above model, we first perform causal segmentation as described in Section~\ref{sec:model:sub:cs}. We compute sets $\textcolor{orange}{X^{out}}$, $\textcolor{green}{X^{in}}$ and $\textcolor{blue}{X^{eq}}$ by varying $\kappa$ in Eq.~(\ref{eq:opt:general}). Then we compute the full solution path and choose subsets based on \textit{information score} $\frac{d\, I(T^{n-1} \to Y^{n})}{d\, \kappa}$ (slope of the red curve in Figure~\ref{pic:solpath}), evaluated for every variable $X_i$ at the point where this variable becomes non-zero (note that $T_i = D_i + \xi_i$). A threshold for the information score is obtained from repeated experiments with uncorrelated $X^{n},Y^{n}$.

Subsequently, we perform the recovery of the causal bipartite graph between $X^n$ and $Y^n$. We compute sets $X_{i}^{out}$, $Y_{i}^{out}$, $X^{eq}$ and $Y^{eq}$ according to the procedure described in Section~\ref{sec:model:sub:bipartite} and add the corresponding edges and arrows to the bipartite graph. We are able to both perform the causal segmentation and recover the bipartite graph correctly, as presented in Figure~\ref{pic:solpath}.

\begin{figure*}[htbp]
	\centering
	\frame{\includegraphics[width = 0.42\textwidth]{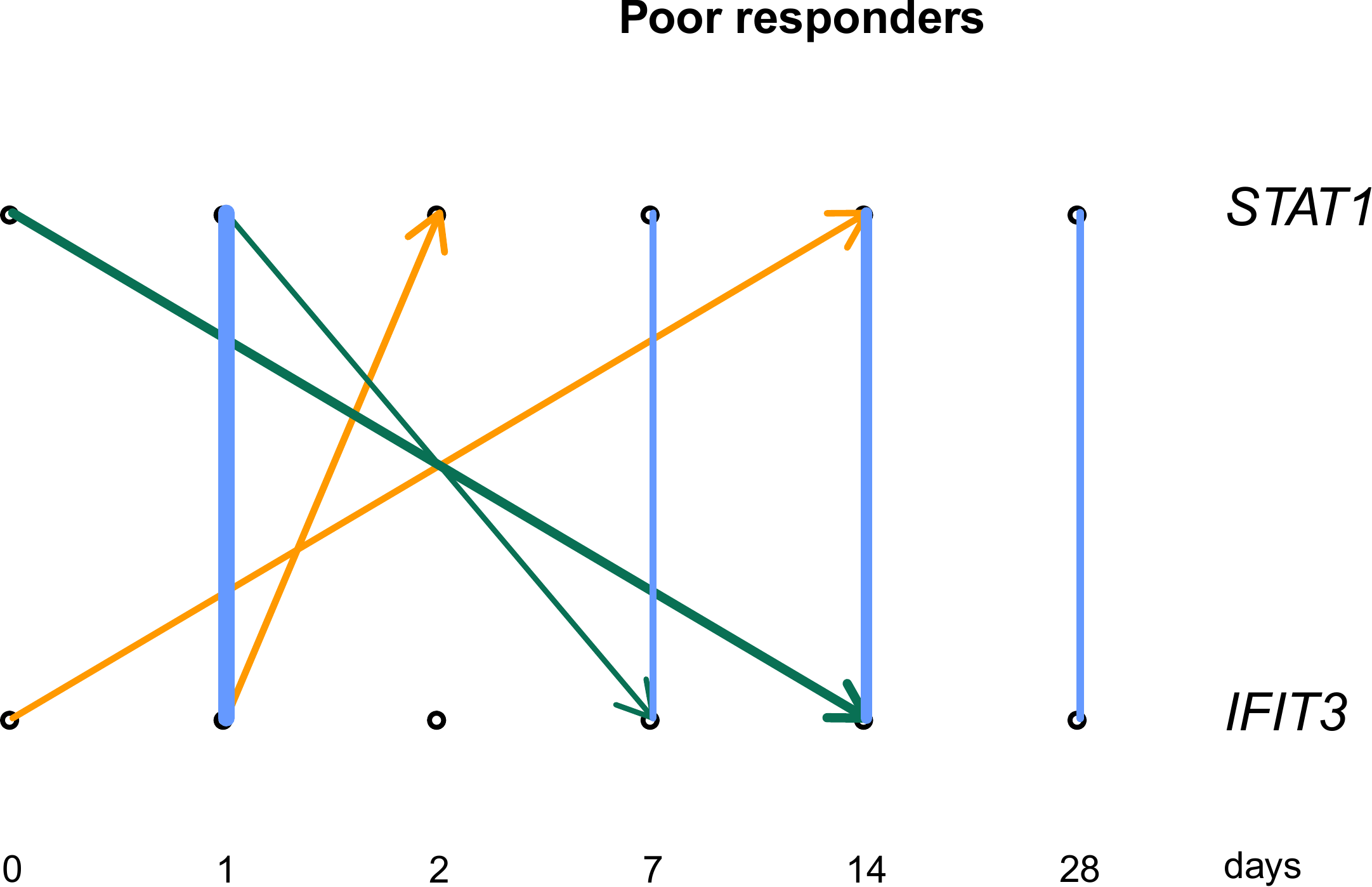}} 
	\quad \quad 
	\frame{\includegraphics[width = 0.42\textwidth]{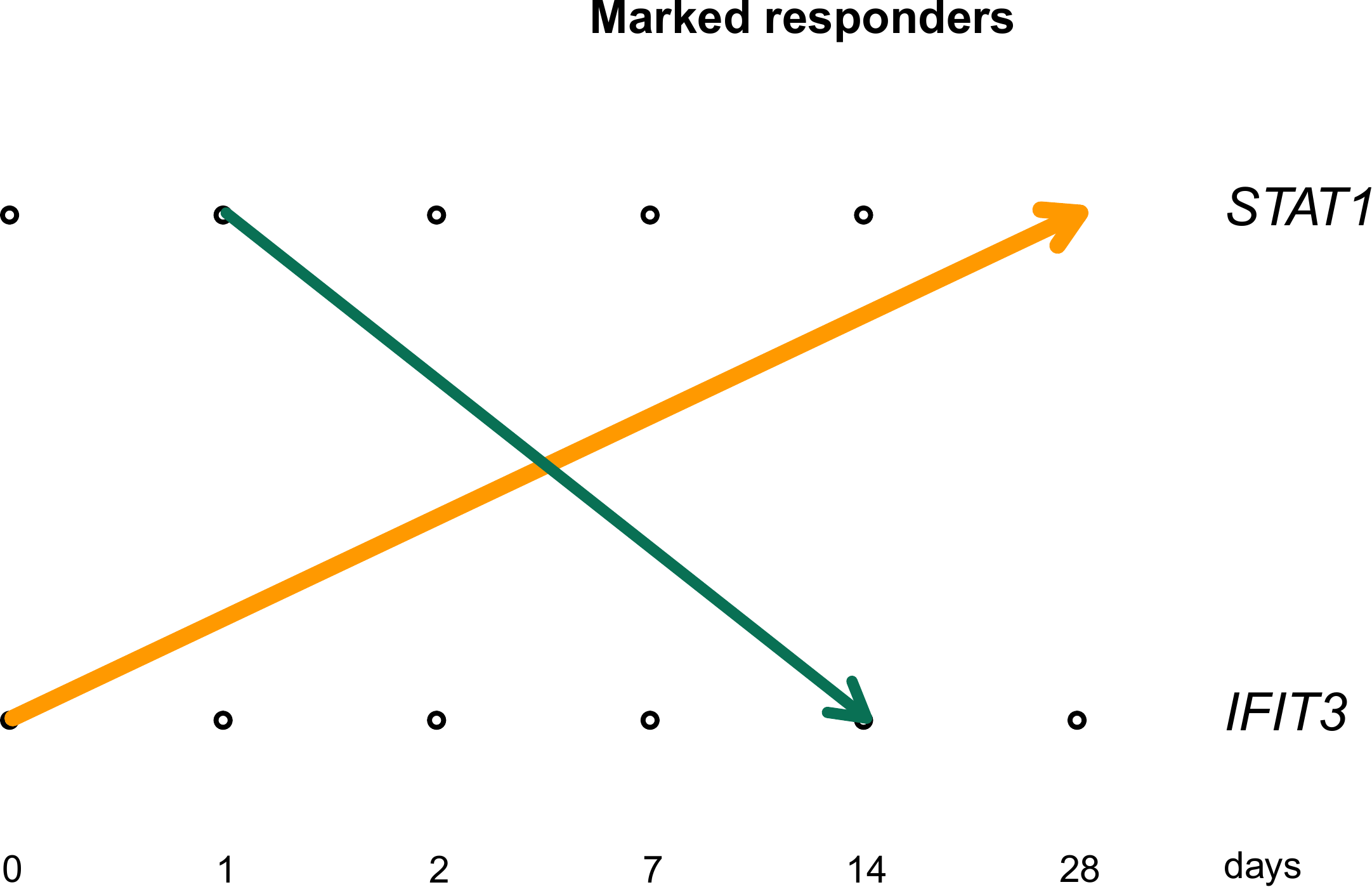} }
	\caption{Time-resolved gene expression data from HCV patients: reconstructed causal graphs for the groups of poor and marked responders. The drawing style and corresponding semantic interpretation is the same as in Figure \ref{pic:solpath}.}
	\label{pic:STAT1}
\end{figure*}

\begin{figure*}[htbp]
  \centering
\vspace{-3mm}
   \fbox{\includegraphics[width = 0.99\textwidth]{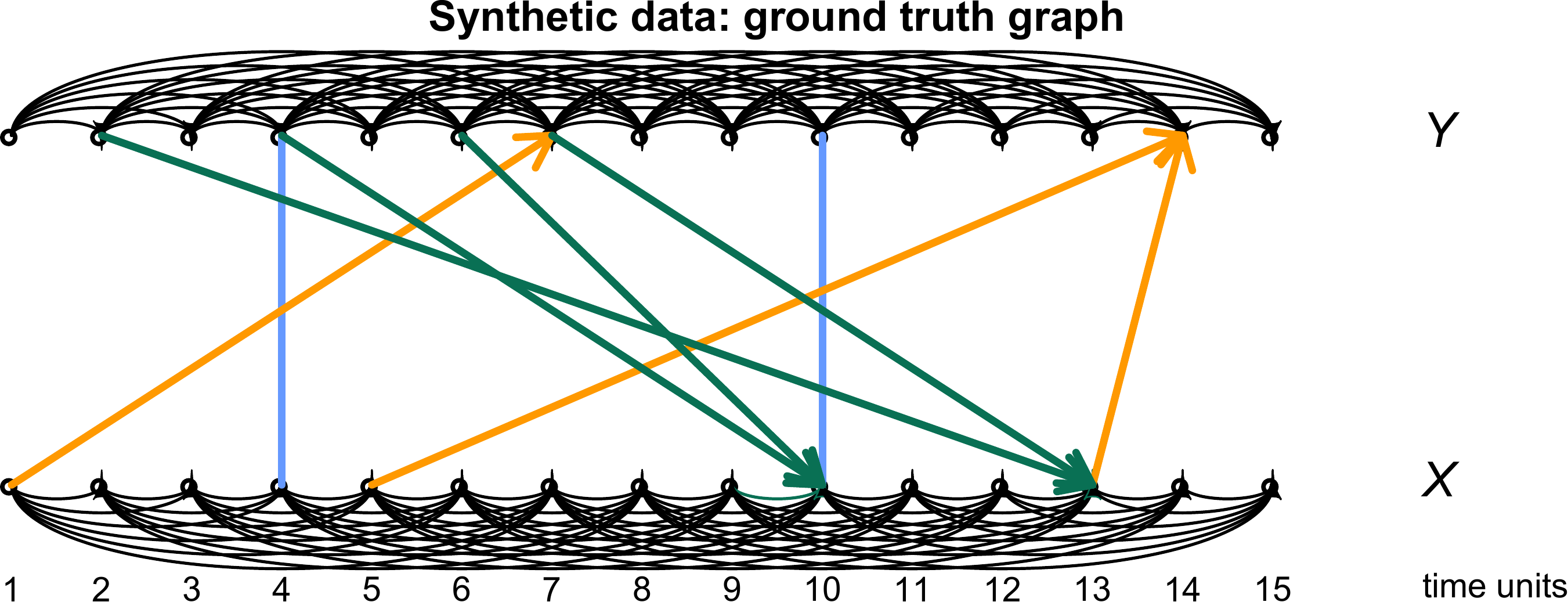}}\\
  \includegraphics[width = 0.475\textwidth]{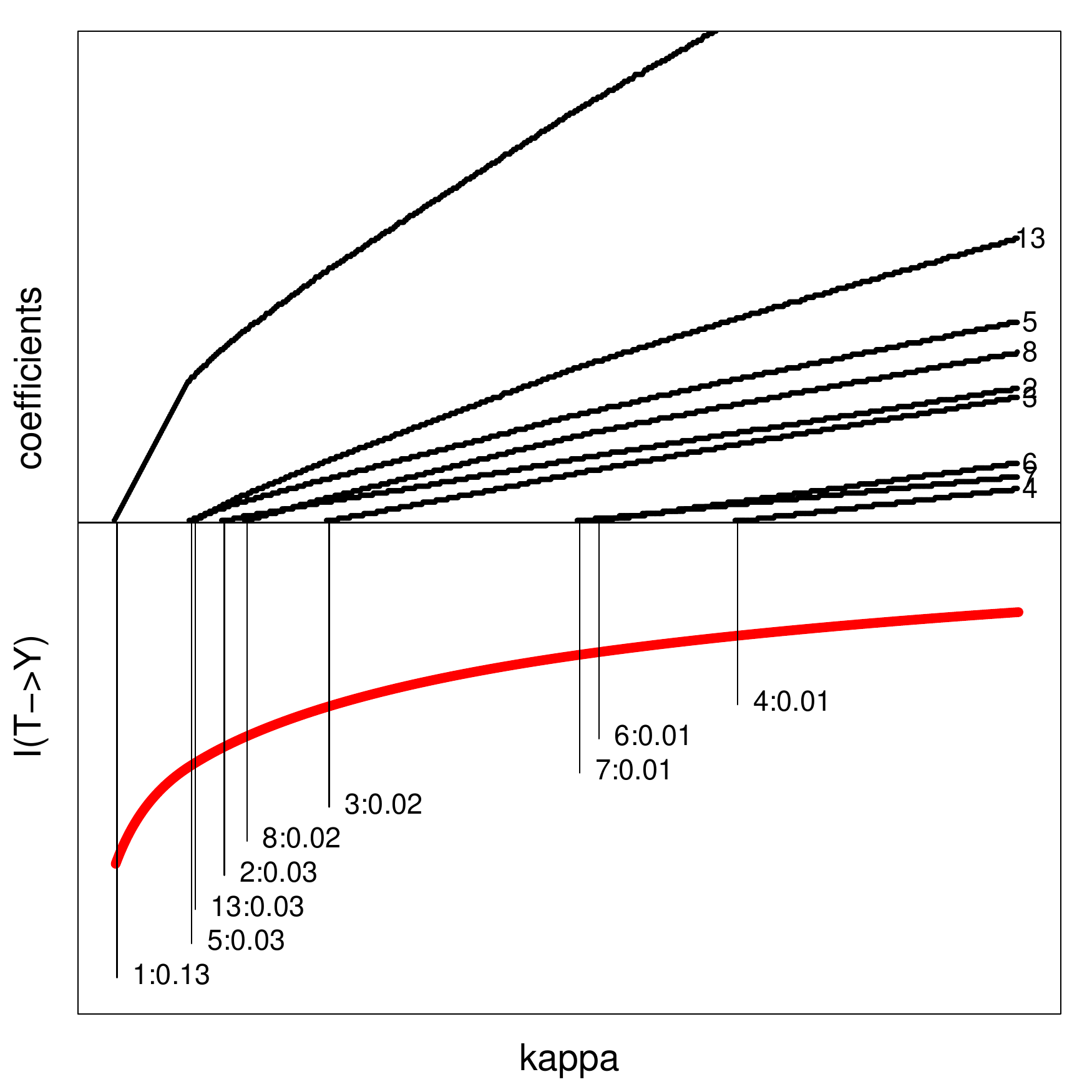} \includegraphics[width = 0.425\textwidth]{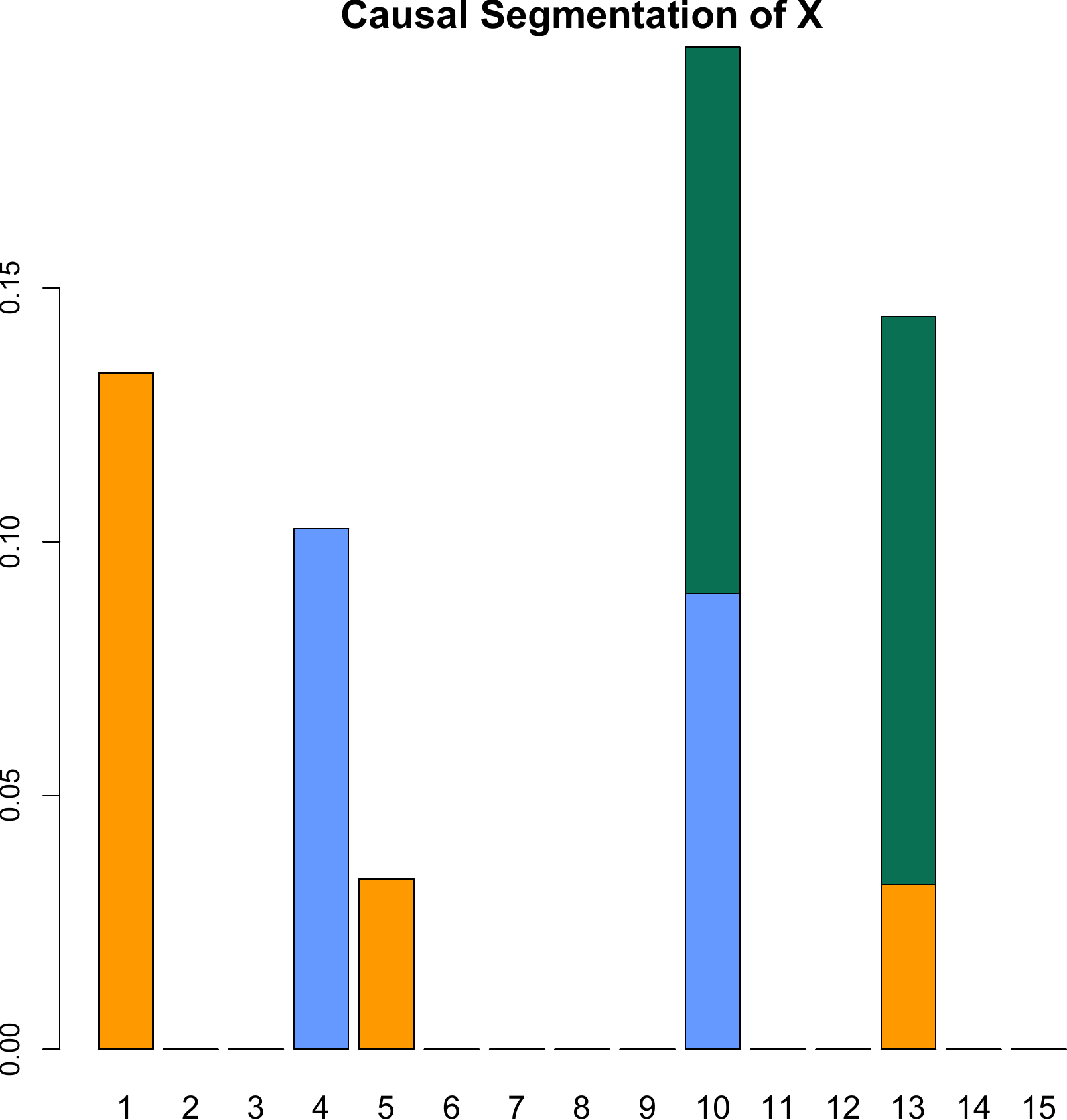}\\
 \fbox{\includegraphics[width = 0.99\textwidth]{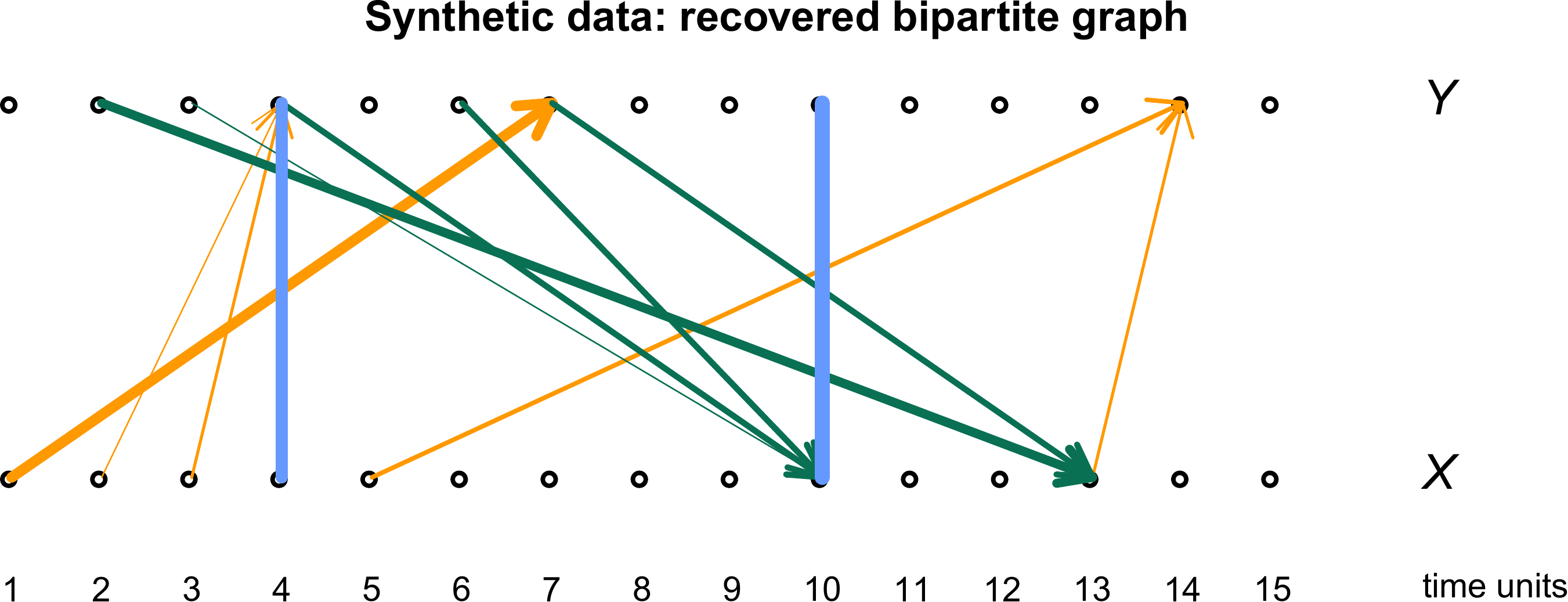}}
  \caption{\textbf{Top:} Ground truth graph. \textbf{Middle:} solution path, information curve  and ``information score'' defined as the derivative at entry points of new variables (left) and causal segmentation of time series $X$ (right). Height of the coloured bars represents the information score of the corresponding coefficient, colour refers to direction: red = ``outgoing'', green = ``incoming'', blue = instantaneous (i.e.~undirected). \textbf{Bottom:} recovered bipartite graph.}
\label{pic:solpath}
\end{figure*}

\subsection{Time resolved gene expression data}
\label{sec:experiments:sub:real}

As a demo application of the causal compression principle we have chosen a human hepatitis C virus (HCV) dataset  that contains time-resolved gene expression profiles from patients with chronic HCV genotype 1 infection \citep{taylor2007changes}. Gene expression was profiled with a HG-U133A GeneChip at six time points after initiation of treatment with pegylated alpha interferon and ribavirin (at days 0,1,2,7,14,28, with ``0'' indicating pre-treatment conditions). Based on the observed decrease in HCV RNA levels at day 28, patients were labelled to have a  ``marked''  (27 patients) or ``poor'' response to treatment (25 patients). The data is available from NCBI/GEO under accession no.~GSE7123. For our analysis, we focused on two different genes that are known to have a crucial interacting role in interferon signalling, namely the transcription factor \emph{STAT1} and the interferon-induced antiviral gene \emph{IFIT3}. Note that these transcriptional interactions between genes take place at timescales of the order of hours, which would appear as \textit{instantaneous couplings} in our dataset with its timescales form days to weeks. We used the same experimental setup as described for synthetic data in Section~\ref{sec:experiments:sub:synthetic}, i.e.~the causal compression principle for the reconstruction of causal bipartite graphs. The analysis was carried out separately for the ``marked'' and the ``poor'' responders, see Figure \ref{pic:STAT1}. There are pronounced differences between the two groups: generally, in the marked responders, the interferon therapy destroys most of the normally tight interactions between \emph{STAT1}   and \emph{IFIT3} (complete loss of instantaneous coupling terms), whereas these interactions seem to be largely unaffected in the poor responders. Secondly, both groups show causal pre-treatment/post-treatment interactions, but for the marked responders,  the influence of initial \emph{IFIT3} on late \emph{STAT1} values is much more prominent. This latter observation might be particularly interesting, since pre-treatment expression levels of interferon-induced genes are known to be strong predictors of treatment response \citep{dill2011interferon}, but the underlying mechanism of this effect is largely unknown.















\section{Conclusion}
\label{sec:conclusion}

We have proposed a new way of discovering causal relationships in temporal data by employing \textbf{causal compression}. We have introduced the chain rule for directed information and proved that causal compression is equivalent to sparsity. Conditions under which the principle of causal compression can be employed have been identified.


We have demonstrated how to tune the compression procedure for the case of time series distributed with a Gaussian copula. A method of causal time series segmentation with respect to incoming and outgoing causal flow as well as instantaneous coupling, was proposed in Section~\ref{sec:model:sub:cs}. Recovery of causal interactions between two time series in the form of a directed bipartite graph was described in Section~\ref{sec:model:sub:bipartite}. Note that the causal compression principle remains valid for arbitrary Pearlian graphs other than time series and non-Gaussian data as long as the directed information can be computed.



The third contribution of the paper is the proposition that directed information can be expressed as a function of the entropy of the copula density only, as stated in Theorem~\ref{thm:1}. In the case of Gaussian distribution this means that one only has to estimate the correlation matrices from the data. This means that the modelling of causality in the framework of Pearlian graphs only requires knowing the copula structure of the modelled data and is independent of their marginals.

\small{
\bibliographystyle{plainnat}
    \bibliography{dIB}
}


\end{document}